\documentclass{article}
\usepackage{microtype}
\usepackage{graphicx}
\usepackage{subfigure}
\usepackage{booktabs}
\usepackage{hyperref}
\usepackage{xcolor}
\usepackage{multirow}
\usepackage{fullpage}

\usepackage{amssymb,amsmath,amsthm}
\usepackage{empheq}
\usepackage{algorithm,algorithmic}

\newcommand{\BBound}{\kappa_B}
\newcommand{\PertBound}{W}
\newcommand{\Diameter}{D}

\newcommand{\CostGradBound}{G}
\newcommand{\hor}{H}
\newcommand{\fixedK}{\mathbb{K}}

\newcommand{\BigBound}{D}
\newcommand{\dimn}{d}

\newcommand{\mincovar}{\sigma^2}
\newcommand{\scmaincost}{\alpha}
\newcommand{\Glb}{\frac{1}{4\kappa^4}}

\def\M{{\mathcal M}}

\def\D{{\mathcal D}}
\def\K{{\mathcal K}}

\def\reals{{\mathbb R}}

\def\norm#1{\mathopen\| #1 \mathclose\|}

\newcommand{\poly}{\mathop{\mbox{\rm poly}}}

\newcommand{\ignore}[1]{}

\def\reals{{\mathbb R}}

\def\bold0{\mathbf{0}}

\def\epsilon{\varepsilon}

\def\grad{\nabla}
\newcommand{\defeq}{\triangleq}

\newtheorem{theorem}{Theorem}[section]

\newtheorem{claim}[theorem]{Claim}

\newtheorem{lemma}[theorem]{Lemma}
\newtheorem{corollary}[theorem]{Corollary}

\newtheorem{definition}[theorem]{Definition}

\newtheorem{assumption}[theorem]{Assumption}

\newcommand{\newreptheorem}[2]{%
\newenvironment{rep#1}[1]{%
 \def\rep@title{#2 \ref{##1}}%
 \begin{rep@theorem}}%
 {\end{rep@theorem}}}

\newreptheorem{theorem}{Theorem}
\newreptheorem{lemma}{Lemma}
\newreptheorem{proposition}{Proposition}
\newreptheorem{claim}{Claim}
\newreptheorem{corollary}{Corollary}
\newreptheorem{mainlemma}{Main Lemma}

\newcommand{\namedref}[2]{\mbox{\hyperref[#2]{#1~\ref*{#2}}}}

\newcommand{\figurerefb}[2]{\mbox{\hyperref[#1]{Figure~\ref*{#1}#2}}}

\newcommand{\equationref}[1]{\mbox{\hyperref[#1]{(\ref*{#1})}}}
\renewcommand{\eqref}{\equationref}

\numberwithin{equation}{section}

\newcommand{\braces}[1]{\left\{#1\right\}}
\newcommand{\pa}[1]{\left(#1\right)}

\newcommand{\abs}[1]{\left|#1\right|}

\def\memdiam{D}
\def\memgradbound{G_f}
\newif\ificmlpaper
\icmlpapertrue
\newif\ifarxiv
\arxivtrue
\begin{document}
\title{Logarithmic Regret for Online Control}
\author{
  Naman Agarwal$^{1}$ \qquad Elad Hazan$^{1\,2}$ \qquad Karan Singh$^{1\,2}$\\
  \\
  $^1$ Google AI Princeton \\
  $^2$ Department of Computer Science, Princeton University \\
  \texttt{namanagarwal@google.com}, \texttt{\{ehazan,karans\}@princeton.edu}\\
}\maketitle
\begin{abstract}
We study optimal regret bounds for control in linear dynamical systems under adversarially changing strongly convex cost functions, given the knowledge of transition dynamics. This includes several well studied and fundamental frameworks such as the Kalman filter and the linear quadratic regulator. State of the art methods achieve regret which scales as $O(\sqrt{T})$, where $T$ is the time horizon. 

We show that the optimal regret in this setting can be significantly smaller, scaling as $O(\poly(\log T))$. This regret bound is achieved by two different efficient iterative methods, online gradient descent and online natural gradient. 
\end{abstract}

\section{Introduction}

Algorithms for regret minimization typically attain one of two performance guarantees. For general convex losses, regret scales as square root of the number of iterations, and this is tight. However, if the loss function exhibit more curvature, such as quadratic loss functions, there exist algorithms that attain poly-logarithmic regret. This distinction is also known as ``fast rates" in statistical estimation. 

Despite their ubiquitous use in online learning and statistical estimation, logarithmic regret algorithms are almost non-existent in control of dynamical systems. This can be attributed to fundamental challenges in computing the optimal controller in the presence of noise. 

Time-varying cost functions in dynamical systems can be used to model unpredictable dynamic resource constraints, and the tracking of a desired sequence of exogenous states. At a pinch, if we have changing (even, strongly) convex loss functions, the optimal controller for a linear dynamical system is not immediately computable via a convex program. For the special case of quadratic loss, some previous works~\cite{cohen2018online} remedy the situation by taking a semi-definite relaxation, and thereby obtain a controller which has provable guarantees on regret and computational requirements. However, this semi-definite relaxation reduces the problem to regret minimization over linear costs, and removes the curvature which is necessary to obtain logarithmic regret. 

In this paper we give the first efficient poly-logarithmic regret algorithms for controlling a linear dynamical system with noise in the dynamics (i.e. the standard model). Our results apply to general convex loss functions that are strongly convex, and not only to quadratics.

\begin{center}
\begin{tabular}{|c|c|c|c|}
\hline
Reference & Noise  &  Regret & loss functions \\
\hline
\hline
\cite{abbasi2014tracking} & {\bf none} & $O(\log^2 {T})$ & quadratic (fixed hessian) \\
\hline
\cite{agarwal2019online} & adversarial & $O(\sqrt{T})$ & convex \\
\hline
\cite{cohen2018online} & stochastic & $O(\sqrt{T})$ & quadratic \\
\hline
{\bf here } & stochastic & $O(\log^7 {T})$ & strongly convex \\
\hline
\end{tabular}
\end{center}

\subsection{Our Results}

The setting we consider is a linear dynamical system, a continuous state Markov decision process with linear transitions, described by the following equation: 
\begin{equation}
\label{eqn:lds}
   x_{t+1} = Ax_t + Bu_t + w_t. 
\end{equation}
Here $x_t$ is the state of the system, $u_t$ is the action (or control) taken by the controller, and $w_t$ is the noise. 
In each round $t$, the learner outputs an action $u_t$ upon observing the state $x_t$ and incurs a cost of $c_t(x_t, u_t)$, where $c_t$ is convex. The objective here is to choose a sequence of adaptive controls $u_t$ so that a minimum total cost may be incurred.

The approach taken by \cite{cohen2018online} and other previous works is to use a semi-definite relaxation for the controller. However, this removes the properties associated with the curvature of the loss functions, by reducing the problem to an instance of online linear optimization. It is known that without curvature, $O(\sqrt{T})$ regret bounds are tight (see \cite{OCObook}).  

Therefore we take a different approach, initiated by \cite{agarwal2019online}. We consider controllers that depend on the previous noise terms, and take the form $u_t = \sum_{i=1}^H M_i w_{t-i}$. While this resulting convex relaxation does not remove the curvature of the loss functions altogether, it results in an overparametrized representation of the controller, and it is not a priori clear that the loss functions are strongly convex with respect to the parameterization. We demonstrate the appropriate conditions on the linear dynamical system under which the strong convexity is retained.

Henceforth we present two methods that attain poly-logarithmic regret. They differ in terms of the regret bounds they afford and the computational cost of their execution. The online gradient descent update (OGD) requires only gradient computation and update, whereas the online natural gradient (ONG) update, in addition, requires the computation of the preconditioner, which is the expected Gram matrix of the Jacobian, denoted $J$, and its inverse. However, the natural gradient update admits an instance-dependent upper bound on the regret, which while being at least as good as the regret bound on OGD, offers better performance guarantees on benign instances (See Corollary~\ref{cor:1d}, for example).

\begin{center}
\begin{tabular}{|c|c|c|}
\hline
Algorithm  & Update rule (simplified) & Applicability \\
\hline
\hline
OGD  &  $M_{t+1} \leftarrow M_t - \eta_t \nabla f_t(M_t ) $ & $\exists K$, diag $L$ s.t. $A-BK = QLQ^{-1}$ \\
\cline{1-2}
ONG  &  $M_{t+1} \leftarrow M_t - \eta_t (\mathbb{E} [J^\top  J])^{-1} \nabla f_t(M_t ) $ & $\|L\| \leq 1- \delta$, $\|Q\|,\|Q\|^{-1} \leq \kappa$\\
\hline
\end{tabular}
\end{center}


\subsection{Related Work}

For a survey of linear dynamical systems (LDS), as well as learning, prediction and control problems, see \cite{stengel1994optimal}. 
Recently, there has been a renewed interest in learning dynamical systems in the machine learning literature. For fully-observable systems, sample complexity and regret bounds for control (under Gaussian noise) were obtained in \cite{abbasi2011regret,dean2018regret,a2}. The technique of spectral filtering for learning and open-loop control of partially observable systems was introduced and studied in \cite{hazan2017learning, arora2018towards, hazan2018spectral}.  Provable control in the Gaussian noise setting via the policy gradient method was also studied in \cite{fazel2018global}.

The closest work to ours is that of \cite{abbasi2014tracking} and \cite{cohen2018online}, aimed at controlling LDS with adversarial loss functions. The authors in \cite{abbasi2011regret} obtain a $O(\log^2 T)$ regret algorithm for changing quadratic costs (with a fixed hessian), but for dynamical systems that are noise-free. In contrast, our results apply to the full (noisy) LDS setting, which presents the main challenges as discussed before. Cohen et al. \cite{cohen2018online} consider changing quadratic costs with stochastic noise to achieve a $O(\sqrt{T})$ regret bound.  

We make extensive use of techniques from online learning  \cite{cesa2006prediction, shalev2012online, OCObook}.  Of particular interest to our study is the setting of online learning with memory \cite{anava2015online}. We also build upon the recent control work of \cite{agarwal2019online}, who use online learning techniques and convex relaxation to obtain provable bounds for LDS with adversarial perturbations.

\section{Problem Setting}

We consider a linear dynamical system as defined in \eqref{eqn:lds} with costs $c_t(x_t, u_t)$, where $c_t$ is strongly convex. In this paper we assume that the noise $w_t$ is a random variable generated independently at every time step. For any algorithm $\mathcal{A}$, we attribute a cost defined as
\begin{equation*}
J_T(\mathcal{A}) = \mathbb{E}_{\{w_{t}\}}\left[\sum_{t=1}^{T} c_t(x_t,u_t)\right],
\end{equation*}
where $x_{t+1}=Ax_t+Bu_t+w_t$, $u_t=\mathcal{A}(x_1, \dots x_t)$ and $\mathbb{E}_{\{w_t\}}$ represents the expectation over the entire noise sequence. For the rest of the paper we will drop the subscript $\{w_t\}$ from the expectation as it will be the only source of randomness. Overloading notation, we shall use $J_T(K)$ to denote the cost of a linear controller $K$ which chooses the action as $u_t=-Kx_t$.

\paragraph{Assumptions.}

In the paper we assume that $x_1=0$ \footnote{This is only for convenience of presentation. The case with a bounded $x_1$ can be handled similarly.}, as well as the following conditions.

\begin{assumption}\label{ass:BW}
We assume that $\|B\|\leq\BBound$.
Furthermore, the perturbation introduced per time step is bounded, i.i.d, and zero-mean with a lower bounded covariance i.e. 
\[\forall t\;w_t \sim \D_{w}, \mathbb{E}[w_t]= 0,  \mathbb{E}[w_tw_t^{\top}] \succeq \mincovar I \text{ and } \|w_t\| \leq \PertBound\]  
\end{assumption}

While we make the assumption that the noise vectors are bounded with probability 1, we can generalize to the case of sub-gaussian noise by conditioning on the event that none of the noise vectors are ever large. This can be done at an expense of another multiplicative $\log(T)$ factor in the regret. Furthermore we assume the following,

\begin{assumption}\label{a2}
The costs $c_t(x, u)$ are $\scmaincost$-strongly convex. Further, as long as it is guaranteed that $\|x\|, \|u\|\leq \Diameter$, it holds that
\[ \|\nabla_x c_t(x,u)\|, \|\nabla_u c_t(x,u)\|\leq \CostGradBound\Diameter.\]
\end{assumption}

\noindent The class of linear controllers we work with are defined as follows. 

\begin{definition}[Diagonal Strong Stability]
Given a dynamics $(A,B)$, a linear policy/matrix $K$ is $(\kappa, \gamma)$-diagonal strongly stable for real numbers $\kappa \geq 1, \gamma < 1$, if there exists a \textit{complex} diagonal matrix $L$ and a non-singular complex matrix $Q$, such that $A-BK=QLQ^{-1}$ and the following conditions are met:
\begin{enumerate}
    \item The spectral norm of $L$ is strictly smaller than one, i.e., $\|L\| \leq 1-\gamma$.
    \item The controller and the transforming matrices are bounded, i.e., $\|K\|\leq \kappa$ and $ \|Q\|,\|Q^{-1}\| \leq \kappa$.
\end{enumerate}
\end{definition}

The notion of strong stability was introduced by \cite{cohen2018online}. Both strong stability and diagonal strong stability are quantitative measures of the classical notion of stabilizing controllers \footnote{A controller $K$ is stabilizing if the spectral radius of $A - BK \leq 1 - \delta$} that permit a discussion on non-asymptotic regret bounds. We note that an analogous notion for quantification of open-loop stability appears in the work of \cite{hazan2018spectral}. 

On the generality of the diagonal strong stability notion, the following comment may be made: while not all matrices are complex diagonalizable, an exhaustive characterization of $m\times m$ complex diagonal matrices is the existence of $m$ linearly independent eigenvectors; for the later, it suffices, but is not necessary, that a matrix has $m$ distinct eigenvalues (See~\cite{strang1993introduction}). It may be observed that almost all matrices admit distinct eigenvalues, and hence, are complex diagonalizable insofar the complement set admits a zero-measure. By this discussion, \textit{almost all} stabilizing controllers are diagonal strongly stable for some $\kappa, \gamma$. The astute reader may note the departure here from the more general notion -- strongly stability -- in that \textit{all} stabilizing controllers are strongly stable for some choice of parameters.




\paragraph{Regret Formulation.}
Let $\K=\{K: K \text{ is } (\kappa,\gamma)\text{-diagonal strongly stable}\}$. For an algorithm $\mathcal{A}$, the notion of regret we consider is \emph{pseudo-regret}, i.e. the sub-optimality of its cost with respect to the cost for the best linear controller i.e.,
\begin{align*}
    & \texttt{Regret} = J_T(\mathcal{A}) - \min_{K\in \K}J_T(K).
\end{align*}
\section{Preliminaries}

\paragraph{Notation.}
We reserve the letters $x,y$ for states and $u,v$ for actions. 
We denote by $d_x, d_u$ to be the dimensionality of the state and the control space respectively. Let $\dimn = \max(d_x, d_u)$. We reserve capital letters $A,B,K,M$ for matrices associated with the system and the policy. Other capital letters are reserved for universal constants in the paper. We use the shorthand $M_{i:j}$ to denote a subsequence $\{M_i, \dots, M_j\}$. For any matrix $U$, define $U_{\text{vec}}$ to be a flattening of the matrix where we stack the columns upon each other. Further for a collection of matrices $M = \{M^{[i]}\}$, let $M_{vec}$ be the flattening defined by stacking the flattenings of $M^{[i]}$ upon each other. We use $\|x\|_U^2 =x^\top Ux$ to denote the matrix induced norm. The rest of this section provides a recap of the relevant definitions and concepts introduced in \cite{agarwal2019online}.

\subsection{Reference Policy Class}

 For the rest of the paper, we fix a $(\kappa, \gamma)$-diagonally strongly stable matrix $\fixedK$ (The bold notation is to stress that we treat this matrix as fixed and not a parameter). Note that this can be any such matrix and it can be computed via a semi-definite feasibility program \cite{cohen2018online} given the knowledge of the dynamics, before the start of the game. We work with following the class of policies.   

\begin{definition}[Disturbance-Action Policy]\label{defn:policy}
A disturbance-action policy $M=(M^{[0]},\dots,M^{[\hor-1]})$, for horizon $\hor \geq 1$ is defined as the policy which at every time $t$, chooses the recommended action $u_{t}$ at a state $x_t$, defined \footnote{$x_t$ is completely determined given $w_0 \ldots w_{t-1}$. Hence, the use of $x_t$ only serves to ease the burden of presentation.} as
\[ u_t(M) \defeq -\fixedK x_t+\sum_{i=1}^{\hor} M^{[i-1]} w_{t-i}. \]
For notational convenience, here it may be considered that $w_i=0$ for all $i<0$.
\end{definition}

The policy applies a linear transformation to the disturbances observed in the past $H$ steps. Since $(x, u)$ is a linear function of the disturbances in the past under a linear controller $K$, formulating the policy this way can be seen as a relaxation of the class of linear policies. Note that $\fixedK$ is a fixed matrix and is not part of the parameterization of the policy. As was established in \cite{agarwal2019online} (and we include the proof for completeness), with the appropriate choice of parameters, superimposing such a $\fixedK$, to the policy class allows it to approximate any linear policy in terms of the total cost suffered with a finite horizon parameter $H$. 

We refer to the policy played at time $t$ as $M_t = \{M_t^{[i]}\}$ where the subscript $t$ refers to the time index and the superscript $[i-1]$ refers to the action of $M_t$ on $w_{t-i}$. Note that such a policy can be executed because $w_{t-1}$ is perfectly determined on the specification of $x_{t}$ as $w_{t-1}=x_{t}-Ax_{t-1}-Bu_{t-1}$.  

\subsection{Evolution of State}
This section describes the evolution of the state of the linear dynamical system under a non-stationary policy composed of a sequence of $T$ policies, where at each time the policy is specified by $M_t=(M_t^{[0]},\dots,M_t^{[\hor-1]})$. We will use $M_{0:T-1}$ to denote such a non-stationary policy. The following definitions ease the burden of notation.
\begin{enumerate}
    \item Define $\tilde{A}=A-B\fixedK$. $\tilde{A}$ shall be helpful in describing the evolution of state starting from a non-zero state in the absence of disturbances.
    \item For any sequence of matrices $M_{0:H}$, define $\Psi_{i}$ as a linear function that describes the effect of $w_{t-i}$ on the state $x_{t}$, formally defined below.
\end{enumerate}
\begin{definition}
For any sequence of matrices $M_{0:H}$, define the disturbance-state transfer matrix $\Psi_{i}$ for $i \in \{0, 1, \ldots H\}$, to be a function with $h+1$ inputs defined as 
\begin{equation*}
    \Psi_{i}(M_{0:H}) \defeq 
    \tilde{A}^i \mathbf{1}_{i\leq H} + \sum_{j=0}^{H} \tilde{A}^j BM_{H-j}^{[i-j-1]} \mathbf{1}_{i-j \in [1,H]}.
\end{equation*}
\end{definition}
\noindent It will be important to note that $\psi_i$ is a \textbf{linear} function of its argument.

\subsection{Surrogate State and Surrogate Cost}

This section introduces a couple of definitions required to describe our main algorithm. In essence they describe a notion of state, its derivative and the expected cost if the system evolved solely under the past $H$ steps of a non-stationary policy.
\begin{definition}[Surrogate State \& Surrogate Action]
Given a sequence of matrices $M_{0:H+1}$ and $2H$ independent invocations of the random variable $w$ given by $\{w_j \sim \D_{w}\}_{j=0}^{2H-1}$, define the following random variables denoting the surrogate state and the surrogate action:
\begin{align*}
  y(M_{0:H}) &= \sum_{i=0}^{2\hor} \Psi_{i}(M_{0:H}) w_{2H-i-i},\\
  v(M_{0:H+1}) &= -\fixedK y(M_{0:H}) + \sum_{i=1}^H M_{H+1}^{[i-1]} w_{2H-i}.
\end{align*}
When $M$ is the same across all arguments we compress the notation to $y(M)$ and $v(M)$ respectively.
\end{definition}

\begin{definition}[Surrogate Cost]
\label{defn:idealcost}
Define the surrogate cost function $f_t$ to be the cost associated with the surrogate state and the surrogate action defined above, i.e.,
\begin{equation*}
    f_{t}(M_{0:H+1}) = \mathbb{E}\left[c_t(y(M_{0:H}), v(M_{0:H+1}))\right].
\end{equation*}
When $M$ is the same across all arguments we compress the notation to $f_t(M)$.
\end{definition}

\begin{definition}[Jacobian]
\label{def:jacobian}
Let $z(M)=\begin{bmatrix}
  y(M) \\v(M)
  \end{bmatrix}$. Since $y(M), v(M)$ are random linear functions of $M$, $z(M)$ can be reparameterized as $ z(M) = JM_{\text{vec}} = \begin{bmatrix} J_y \\ J_v \end{bmatrix} M_{\text{vec}}$,
where $J$ is a random matrix, which derives its randomness from the random perturbations $w_i$.
\end{definition}

\subsection{OCO with Memory}\label{sec:oco_mem}
We now describe the setting of online convex optimization with memory introduced in \cite{anava2015online}. 
In this setting, at every step $t$, an online player chooses some point $x_t \in \K \subset \reals^d$, a loss function $f_t : \K^{\hor+1}\mapsto \reals$ is then revealed, and the learner suffers a loss of $f_t(x_{t-\hor:t})$. We assume a certain coordinate-wise Lipschitz regularity on $f_t$ of the form such that, for any $j \in \braces{0,\dots,\hor}$, for any $x_{0:\hor},\tilde{x}_j\in \K$,
\begin{equation}\label{eq:memlip}
    \abs{f_t(x_{0:j-1},x_j,x_{j+1:\hor}) - f_t(x_{0:j-1},\tilde{x}_j,x_{j+1:\hor})}
    \leq L\norm{x_j - \tilde{x}_j}.
\end{equation}
In addition, we define ${f}_t(x) = f_t(x,\dots,x)$, and we let
\begin{equation}\label{eq:membounds}
    \memgradbound = \sup\limits_{t \in \braces{0,\dots,T}, x \in \K} \norm{\grad {f}_t(x)}, \; \text{  }\;  \memdiam = \sup\limits_{x,y\in \K} \norm{x-y}.
\end{equation}
The resulting goal is to minimize the \emph{policy regret} \cite{arora2012online}, which is defined as
\begin{equation*}
    \texttt{PolicyRegret} = \sum\limits_{t=\hor}^T f_t(x_{t-\hor:t}) - \min\limits_{x \in \K} \sum\limits_{t=H}^T {f}_t(x).
\end{equation*}


\section{Algorithms \& Statement of Results}
\label{sec:main}

The two variants of our method are spelled out in Algorithm \ref{alg:mainA}. Theorems \ref{thm:main} and \ref{thm:main_2} provide the main guarantees for the two algorithms. 

\begin{algorithm}[t!]
\caption{Online Control Algorithm} 
\label{alg:mainA}
\begin{algorithmic}[1]
\STATE \textbf{Input:} Step size schedule $\eta_t$, Parameters $\BBound, \kappa, \gamma, T$.
\STATE Define $\hor = \gamma^{-1}\log(T\kappa^2)$
\STATE Define $\M = \{M = \{M^{[0]} \ldots M^{[\hor-1]}\}: \|M^{[i-1]}\| \leq \kappa^3 \BBound (1 - \gamma)^i\}$.
\STATE Initialize $M_0\in \M$ arbitrarily.
\FOR{$t = 0, \ldots, T-1$}
\STATE Choose the action: \vspace{-0.3cm}\[u_t = -\fixedK x_t+ \sum_{i=1}^{\hor} M_t^{[i-1]} w_{t-i}.\]\vspace{-0.4cm}
\STATE Observe the new state $x_{t+1}$ and record $w_t=x_{t+1}-Ax_t-Bu_t$.
\STATE \textbf{Online Gradient Update:} 
\[M_{t+1} = \Pi_{\M}(M_t - \eta_t \nabla f_{t}(M_t)) \]
\STATE \textbf{Online Natural Gradient Update:} 
\[M_{vec,t+1} = \Pi_{\M}(M_{vec,t} - \eta_t (\mathbb{E} [J^TJ])^{-1}\nabla_{M_{vec, t}} f_{t}(M_t))\]
\ENDFOR
\end{algorithmic}
\end{algorithm}

\paragraph{Online Gradient Update}

\begin{theorem}[Online Gradient Update]\label{thm:main}
Suppose Algorithm~\ref{alg:mainA} (Online Gradient Update) is executed with $\fixedK$ being any $(\kappa, \gamma)$-diagonal strongly stable matrix and $\eta_t=\Theta\left(\scmaincost  \mincovar t\right)^{-1}$, on an LDS satisfying Assumption \ref{ass:BW} with control costs satisfying Assumption \ref{a2}. Then, it holds true that
\[J_T(\mathcal{A}) - \min_{K\in \K}J_T(K) \leq \tilde{O}\left(\frac{\CostGradBound^2 \PertBound^4}{\scmaincost \mincovar}  \log^7(T)\right).\]
\end{theorem}

The above result leverages the following lemma which shows that the function $f_t(\cdot)$ is strongly convex with respect to its argument $M$. Note that strong convexity of the cost functions $c_t$ over the state-action space does not by itself imply the strong convexity of the surrogate cost $f_t$ over the space of controllers $M$. This is because, in the surrogate cost $f_t$,  $c_t$ is applied to $y(M), v(M)$ which themselves are linear functions of $M$; the linear map $M$ is necessarily column-rank-deficient. To observe this, note that $M$ maps from a space of dimensionality $H\times \textrm{dim}(x)\times \textrm{dim}(u)$ to that of $\textrm{dim}(x)+\textrm{dim}(u)$. The next theorem, which forms the core of our analysis, shows that this is not the case using the inherent stochastic nature of the dynamical system.

\begin{lemma}
\label{lem:sc}
If the cost functions $c_t(\cdot, \cdot)$ are $\alpha$-strongly convex, $\fixedK$ is a $(\kappa, \gamma)$ diagonal strongly stable matrix and Assumption \ref{ass:BW} is met then the idealized functions $f_t(M)$ are $\lambda$-strongly convex with respect to $M$ where 
    \[\lambda = \frac{\alpha\sigma^2 \gamma^2}{36\kappa^{10}} \]
\end{lemma}

\noindent We present the proof of simpler instances, including a one dimensional version of the theorem, in Section \ref{sec:sc1dim}, as they present the core ideas without the tedious notation necessitated by the general setting. We provide the general proof in Section \ref{sec:appsc} of the Appendix.

\paragraph{Online Natural Gradient Update}
\begin{theorem}[Online Natural Gradient Update]\label{thm:main_2}
Suppose Algorithm~\ref{alg:mainA} (Online Natural Gradient Update) is executed with $\eta_t=\Theta\left(\scmaincost t\right)^{-1}$, on an LDS satisfying Assumptions \ref{ass:BW} and with control costs satisfying Assumption \ref{a2}. Then, it holds true that
\[J_T(\mathcal{A}) - \min_{K\in \K}J_T(K) \leq \tilde{O}\left(\frac{\CostGradBound \PertBound^2}{\scmaincost  \mu}  \log^7(T)\right) \quad \text{where} \quad \mu^{-1} \defeq \max_{M\in \mathcal{M}} \|(\mathbb{E}[J^TJ])^{-1}\nabla_{M_{\text{vec}}} f_t(M)\|.\]
\end{theorem}

\noindent In Theorem \ref{thm:main_2}, the regret guarantee depends on an instance-dependent parameter $\mu$, which is a measure of hardness of the problem. First, we note that the proof of Lemma \ref{lem:sc} establishes that the Gram matrix of the Jacobian (Defintion \ref{def:jacobian}) is strictly positive definite and hence we recover the logarithmic regret guarantee achieved by the Online Gradient Descent Update, with the constants preserved. 

\begin{corollary}\label{cor:sph}
In addition to the assumptions in Theorem~\ref{thm:main_2}, if $\fixedK$ is a $(\kappa, \gamma)$-diagonal strongly stable matrix, then for the natural gradient update
\[J_T(\mathcal{A}) - \min_{K\in \K}J_T(K) \leq \tilde{O}\left(\frac{\CostGradBound^2 \PertBound^4}{\scmaincost \mincovar}  \log^7(T)\right),\]
\end{corollary}
\begin{proof}
The conclusion follows from Lemma~\ref{lem:gdbound} and Lemma \ref{lemma:mainjacob} which is the core component in the proof of Lemma \ref{lem:sc} showing that $\mathbb{E}[J^T J] \geq  \frac{ \gamma^2\sigma^2}{36\kappa^{10}}  \cdot \mathbb{I}$ .
\end{proof}

\noindent Secondly, we note that, being instance-dependent, the guarantee the Natural Gradient update offers can potentially be stronger than that of the Online Gradient method. A case in point is the following corollary involving spherically symmetric quadratic costs, in which case the Natural Gradient update yields a regret guarantee under demonstrably more general conditions, in that the bound does not depend on the minimum eigenvalue of the covariance of the disturbances $\sigma^2$, unlike the one OGD affords \footnote{A more thorough analysis of the improvement in this case shows a multiplicative gain of $\frac{WDH\sqrt{d}\kappa^{10}}{\sigma^2\gamma^2}$. Furthermore, Theorem~\ref{thm:main_2} and Corollary~\ref{cor:sph} hold more generally under strong stability of the comparator class and $\fixedK$, as opposed to diagonal strong stability.}.

\begin{corollary}\label{cor:1d}
Under the assumptions on Theorem~\ref{thm:main_2}, if the cost functions are of the form $c_t(x,u) = r_t (\|x\|^2+\|u\|^2)$, where $r_t\in [\alpha, \beta]$ is an adversarially chosen sequence of numbers and $\fixedK$ is chosen to be a $(\kappa, \gamma)$-diagonal strongly stable matrix, then the natural gradient update guarantees \[J_T(\mathcal{A}) - \min_{K\in \K}J_T(K) \leq \tilde{O}\left(\frac{\beta^2 \PertBound^2}{\scmaincost}  \log^7(T)\right),\]
\end{corollary}
\begin{proof}
It suffices to note $\|\nabla_{M_{\text{vec}}} f_t(M)\|_{(\mathbb{E}[J^TJ])^{-2}}  = \|\mathbb{E}[J^T(r_t \cdot  I)JM_{\text{vec}}]\|_{(\mathbb{E}[J^TJ])^{-2}} \leq \beta \|M_{\text{vec}}\| $.
\end{proof}

\section{Reduction to Low Regret with Memory}
The next lemma is a condensation of the results from \cite{agarwal2019online} which we present in this form to highlight the reduction to OCO with memory. It shows that achieving low policy regret on the memory based function $f_t$ is sufficient to ensure low regret on the overall dynamical system. Since the proof is essentially provided by \cite{agarwal2019online}, we provide it in the Appendix for completeness. Define, \[\M \defeq \{M = \{M^{[0]} \ldots M^{[\hor-1]}\}: \|M^{[i-1]}\| \leq \kappa^3 \BBound (1 - \gamma)^i\}.\]

\begin{lemma}
\label{lem:mainreductionlemma}
Let the dynamical system satisfy Assumption \ref{ass:BW} and let $\fixedK$ be any $(\kappa,\gamma)$-diagonal strongly stable matrix. Consider a sequence of loss functions $c_t(x,u)$ satisfying Assumption \ref{a2} and a sequence of policies $M_0 \ldots M_T$ satisfying
\begin{equation*}
    \texttt{PolicyRegret} =\sum_{t=0}^T f_t(M_{t-H-1:t}) - \min_{M \in \M} \sum_{t=0}^{T} f_t(M) \leq R(T)
\end{equation*} 
for some function $R(T)$ and $f_t$ as defined in Definition \ref{defn:idealcost}. Let $A$ be an online algorithm that plays the non-stationary controller sequence $\{M_0,\dots M_T\}$. Then as long as $H$ is chosen to be larger than $\gamma^{-1}\log(T\kappa^2)$ we have that
\begin{align*}
  J(A) - \min_{K^*\in \K}J(K^*)  \leq R(T) + O(GW^2\log(T)),
\end{align*}
Here $O(\cdot)$, $\Theta(\cdot)$ contain polynomial factors in  $\gamma^{-1}, \BBound, \kappa, \dimn$. 
\end{lemma}

\begin{lemma}
  \label{lem:gdbound} The function $f_t$ as defined in Definition~\ref{defn:idealcost} is coordinate-wise $L$-lipschitz and the norm of the gradient is bounded by $G_f$, where
  \[
    L = \frac{2DGW\kappa_B\kappa^3}{\gamma}, \;\; G_f\leq GDWHd\left(H+\frac{2\kappa_B\kappa^3}{\gamma}\right) \]\[\text{where } \BigBound \triangleq \frac{W\kappa^2(1 + \hor\BBound^2 \kappa^3)}{\gamma ( 1 - \kappa^2(1 - \gamma)^{\hor+1})} + \frac{\BBound \kappa^3 W}{\gamma}.\]
\end{lemma}
\noindent The proof of this lemma is identical to the analogous lemma in \cite{agarwal2019online} and hence is omitted.
 
\section{Analysis for Online Gradient Descent}
In the setting of Online Convex Optimization with Memory, as shown by \cite{anava2015online}, by running a memory-based OGD, we can bound the policy regret by the following theorem.
\begin{theorem}\label{thm:oco_memory}
Consider the OCO with memory setting defined in Section \ref{sec:oco_mem}. Let $\braces{f_t}_{t=\hor}^T$ be Lipschitz loss functions with memory such that $f_t(x)$ are $\lambda$-strongly convex, and let $L$ and $\memgradbound$ be as defined in \eqref{eq:memlip} and \eqref{eq:membounds}. Then, there exists an algorithm which generates a sequence $\braces{x_t}_{t=0}^T$ such that
\begin{equation*}
    \sum\limits_{t=\hor}^T f_t(x_{t-\hor:t}) - \min\limits_{x \in \K} \sum\limits_{t=\hor}^T \tilde{f}_t(x) \leq \frac{G_f^2 + L\hor^2G_f}{\lambda}(1 + \log(T)).  
\end{equation*}
\end{theorem}
\noindent We provide the requisite algorithm and the proof of the above theorem in the Appendix. 

\paragraph{Specialization to the Control Setting:} We combine bound the above with the listed reduction.

\begin{proof}[Proof of Theorem~\ref{thm:main}]
Setting $H=\gamma^{-1}\log(T\kappa^2)$, Theorem~\ref{thm:oco_memory}, in conjunction with Lemma~\ref{lem:gdbound}, implies that policy regret is bounded by $\tilde{O}\left(\frac{G^2W^4H^6}{\alpha \sigma^2} \log T\right)$. An invocation of Lemma~\ref{lem:mainreductionlemma} now suffices to conclude the proof of the claim.
\end{proof}

\section{Analysis for Online Natural Gradient Descent}

In this section, we consider structured loss functions of the form $f_t(M_{0:H+1})=\mathbb{E}[c_t(z)]$, where $z=\sum_{i=0}^{H+1} J_i[M_{i}]_{\text{vec}}$. $J_i$ is a random matrix, and $c_t$'s are adversarially chosen strongly convex loss functions. In a similar vein, define $f_t(M)$ to be the specialization of $f_t$ when input the same argument, i.e. $M$, $H+1$ times. Define $J=\sum_{i=0}^{H+1} J_i$. 

The following lemma  provides upper bounds on the regret bound as well as the norm of the movement of iterate at every round for the Online Natural Gradient Update (Algorithm \ref{alg:mainA}).

\begin{lemma}
\label{lem:joc}
For $\alpha$-strongly convex $c_t$, if the iterates $M_t$ are chosen as per the update rule:
\[ [M_{t+1}]_{\text{vec}} = \Pi_{\mathcal{M}}\left({[M_t]_{\text{vec}} - \eta_t (\mathbb{E} [J^TJ])^{-1} \nabla_{[M_t]_{\text{vec}}} f_t(M_t)}\right)\]
with a decreasing step size of $n_t = \frac{1}{\alpha t}$, it holds that
\[\sum_{t=1}^T f_t(M_t)-\min_{M^*\in \mathcal{M}}\sum_{t=1}^Tf_t(M^*) \leq (2\alpha)^{-1} \max_{M\in \mathcal{M}} \|\nabla_{M_{\text{vec}}} f_t(M)\|_{(\mathbb{E}[J^TJ])^{-1}}^2\log T.\]
Moreover, the norm of the movement of consecutive iterates is bounded for all $t$ as
\[ \|[M_{t+1}]_{\text{vec}}-[M_{t}]_{\text{vec}}\| \leq (\alpha t)^{-1} \max_{M\in\mathcal{M}}\|(\mathbb{E} [J^TJ])^{-1}\nabla_{M_{\text{vec}}} f_t(M)\|.\]
\end{lemma}

The following theorem now bounds the total for the online game with memory. 

\begin{theorem}\label{thm:oco_memory_2}
In the setting desribed in this subsection, let $c_t$ be $\alpha$-strongly convex, and $f_T$ be such that it satisfies equation \eqref{eq:memlip} with constant $L$, and $G_f=\max_{M\in\mathcal{M}}\|(\mathbb{E} [J^TJ])^{-1}\nabla_{M_{\text{vec}}} f_t(M)\|$. Then, the online natural gradient update generates a sequence $\braces{M_t}_{t=0}^T$ such that
\begin{equation*}
    \sum\limits_{t=\hor}^T f_t(M_{t-\hor:t}) - \min\limits_{M \in \mathcal{M}} \sum\limits_{t=\hor}^T \tilde{f}_t(M) \leq \frac{\max_{M\in \mathcal{M}} \|\nabla_{M_{\text{vec}}} f_t(M)\|_{(\mathbb{E}[J^TJ])^{-1}}^2 + L\hor^2G_f}{\alpha}(1 + \log(T)).  
\end{equation*}
\end{theorem}

\begin{proof}[Proof of Theorem \ref{thm:oco_memory_2}]
We know by \eqref{eq:memlip} that, for any $t \geq H$,
\begin{align*}
\abs{f_t(M_{t-H:t}) - f_t(M)} &\leq L\sum\limits_{j=1}^H\norm{[M_t]_{\text{vec}} - [M_{t-j}]_{\text{vec}}} \leq L\sum\limits_{j=1}^H\sum\limits_{l=1}^j\norm{[M_{t-l+1}]_{\text{vec}} - [M_{t-l}]_{\text{vec}}}\\
&\leq L\sum\limits_{j=1}^H\sum\limits_{l=1}^j \eta_{t-l}\max_{M\in\mathcal{M}}\|(\mathbb{E} [J^TJ])^{-1}\nabla_{M_{\text{vec}}} f_t(M)\| \\
&\leq LH^2\eta_{t-\hor} \max_{M\in\mathcal{M}}\|(\mathbb{E} [J^TJ])^{-1}\nabla_{M_{\text{vec}}} f_t(M)\|,
\end{align*}
and so we have that
\begin{equation*}
\abs{\sum\limits_{t=H}^T f_t(M_{t-H:t}) - \sum\limits_{t=H}^T f_t(M_t)} \leq \frac{LH^2 G_f}{\alpha}(1 + \log(T)).
\end{equation*}
The result follows by invoking Lemma~\ref{lem:joc}.
\end{proof}

\paragraph{Specialization to the Control Setting:} We combine bound the above with the listed reduction.

\begin{proof}[Proof of Theorem~\ref{thm:main_2}]
First observe that $\|\nabla_{M_{\text{vec}}} f_t(M)\|_{(\mathbb{E}[J^TJ])^{-1}}^2 \leq \mu^{-1} \|\nabla_{M_{\text{vec}}} f_t(M)\|$. Setting $H=\gamma^{-1}\log(T\kappa^2)$, Theorem~\ref{thm:oco_memory_2}, in conjunction with Lemma~\ref{lem:gdbound}, imply the stated bound on policy regret. An invocation of Lemma~\ref{lem:mainreductionlemma} suffices to conclude the proof of the claim.
\end{proof}

\section{Proof of Strong Convexity in simpler cases}
\label{sec:sc1dim}
In this section we illustrate the proof of strong convexity of the function $f_t(M)$ with respect to $M$, i.e. Lemma~\ref{lem:sc}, in two settings.
\begin{enumerate}
    \item The case when $\fixedK = 0$ is a diagonal strongly stable policy. 
    \item A specialization of Lemma \ref{lem:sc} to one-dimensional state and one-dimensional control.
\end{enumerate}
This latter case highlights the difficulty caused in the proof due to a choosing a non-zero $\fixedK$ and presents the main ideas of the proof without the tedious tensor notations necessary for the general case.

We will need some definitions and preliminaries that are outlined below. By definition we have that $f_t(M) = \mathbb{E}[c_t(y_t(M),v_t(M))]$. Since we know that $c_t$ is strongly convex we have that
\[ \nabla^2 f_t(M) = \mathbb{E}_{\{w_k\}_{k=0}^{2H-1}}[\nabla^2 c_t(y(M), v(M))] \succeq \alpha  \mathbb{E}_{\{w_k\}_{k=0}^{2H-1}}[J_{y}^{\top}J_y + J_{v}^{\top}J_v]. \]
We remind the reader that $J_y, J_v$ are random matrices dependent on the noise vectors $\{w_k\}_{k=0}^{2H-1}$. In each of the above cases, we will demonstrate the truth of the following lemma implying Lemma \ref{lem:sc}.

\begin{lemma}
\label{lemma:mainjacob}
If Assumption \ref{ass:BW} is satisfied and $\fixedK$ is chosen to be a $(\kappa, \gamma)$-diagonal strongly stable matrix, then the following holds,
\begin{align*}
    \mathbb{E}_{\{w_k\}_{k=0}^{2H-1}}[J_y^{\top}J_y + J_{v}^{\top}J_v] \succeq  \frac{ \gamma^2 \sigma^2 }{36\kappa^{10}}  \cdot \mathbb{I}.
\end{align*}
\end{lemma}
To analyze $J_y, J_v$, we will need to rearrange the definition of $y(M)$ to make the dependence on each individual $M^{[i]}$ explicit. To this end consider the following definition for all $k \in [\hor+1]$.
\[\tilde{v}_{k}(M) \defeq \sum_{i=1}^{\hor}M^{[i-1]}w_{2\hor - i - k}\]
Under this definition it follows that
\[y(M) = \sum_{k=1}^{\hor} (A-B\fixedK)^{k-1}B\tilde{v}_k(M) + \sum_{k=1}^{\hor}(A-B\fixedK)^{k-1}w_{2H-k}\]
\[v(M) = -\fixedK y(M) + \tilde{v}_0(M)\]
From the above definitions, $(J_y,J_v)$ may be characterized in terms of the Jacobian of $\tilde{v}_k$ with respect to $M$, which we define for the rest of the section as $J_{\tilde{v}_k}$. Defining $M_{\mathrm{vec}}$ as the stacking of rows of each $M^{[i]}$ vertically, i.e. stacking the columns of $(M^{[i]})^{\top}$, it can be observed that for all $k$,
\[ J_{\tilde{v}_k} = \frac{\partial \tilde{v}_k(M)}{ \partial M} = \left[ I_{d_u} \otimes w_{2\hor-k-1}^{\top} \;\;\;\;I_{d_u} \otimes w_{2\hor-k-2}^{\top} \;\;\ldots\;\; I_{d_u} \otimes w_{\hor-k}^{\top}\right]\]
where $d_u$ is the dimension of the controls. We are now ready to analyze the two simpler cases. Further on in the section we drop the subscripts $\{w_k\}_{k=0}^{2H-1}$ from the expectations for brevity. 

\subsection{Proof of Lemma \ref{lemma:mainjacob}: $\fixedK=0$}
In this section we assume that $\fixedK = 0$ is a $(\kappa, \gamma)$-diagonal strongly stable policy for $(A,B)$. Be definition, we have $v(M) = \tilde{v}_0(M)$.
One may conclude the proof with the following observation.
\[\mathbb{E}[J_y^{\top}J_y + J_{v}^{\top}J_v] \succeq \mathbb{E}[ J_{v}^{\top}J_v]  = \mathbb{E}[ J_{\tilde{v}_0}^{\top}J_{\tilde{v}_0}] = I_{d_u}\otimes \Sigma \succeq \sigma^2\mathbb{I}.\]

\subsection{Proof of Lemma \ref{lemma:mainjacob}: 1-dimensional case}

Note that in the one dimensional case, the policy given by $M = \{M^{[i]}\}_{i=0}^{H-1}$ is an $H$ dimensional vector with $M^{[i]}$ being a scalar. Furthermore $y(M), v(M), \tilde{v}_k(M)$ are scalars and hence their Jacobians $J_y, J_v, J_{\tilde{v}_k}$ with respect to $M$ are $1 \times H$ vectors. In particular we have that,
\[ J_{\tilde{v}_k} = \frac{\partial \tilde{v}_k(M)}{ \partial M} = \left[ w_{2\hor-k-1} \;\;\;\; w_{2\hor-k-2} \;\;\ldots\;\; w_{\hor-k}\right]\]
\noindent Therefore using the fact that $E[w_iw_j] = 0$ for $i\neq j$ and $\mathbb{E}[w_i^2] = \sigma^2$, it can be observed that for any $k_1,k_2$, we have that
\begin{equation}
\label{eqn:mainjacob1}
  \mathbb{E}[J_{\tilde{v_{k_1}}}^{\top}J_{\tilde{v_{k_2}}}] = \mathcal{T}_{k_1 - k_2} \cdot \sigma^2  
\end{equation}
where $\mathcal{T}_m$ is defined as an $\hor \times \hor$ matrix with $[\mathcal{T}_m]_{ij} = 1$ if and only if $i-j = m$ and $0$ otherwise. This in particular immediately gives us that,
\begin{align}
    \label{eqn:jacoby1}
    \mathbb{E}[J_y^{\top}J_{y}] &= \underbrace{\left( \sum_{k_1=1}^{\hor}\sum_{k_2=1}^{\hor} \mathcal{T}_{k_1 - k_2} \cdot (A-B\fixedK)^{k_1-1 + k_2 - 1}\right)}_{\defeq \mathbb{G}}\cdot  B^2\cdot\sigma^2\\    \mathbb{E}[J_{\tilde{v_0}}^{\top}J_y] &=
    \underbrace{\left(\sum_{k=1}^{\hor} \mathcal{T}_{-k} (A-B\fixedK)^{k-1}\right)}_{\defeq \mathbb{Y}}\cdot B \cdot \sigma^2 
\end{align}
First, we prove a few spectral properties of the matrices $\mathbb{G}$ and $\mathbb{Y}$ defined above. From Gershgorin's circle theorem, and the fact that $\fixedK$ is $(\kappa, \gamma)$-diagonal strongly stable, we have
\begin{equation}
\label{eqn:Ynorm1dim}
\|\mathbb{Y}+\mathbb{Y}^{\top}\| \leq \|\sum_{k=1}^{\hor}(\mathcal{T}_{-k} + \mathcal{T}_{k})(A-B\fixedK)^{k-1}\| \leq 2\gamma^{-1}
\end{equation}
The spectral properties of $\mathbb{G}$ summarized in the lemma below form the core of our analysis.
\begin{lemma}
\label{lemma:GIlemma1dim}
$\mathbb{G}$ is a symmetric positive definite matrix. In particular
    \[\mathbb{G} \succeq \frac{1}{4} \cdot I.\]
\end{lemma}
\noindent Now consider the statements which follow by the respective definitions. 
\begin{align*}
    \mathbb{E}[J_{v}^{\top}J_v] &= \mathbb{K}^2\cdot \mathbb{E}[J_{y}^{\top}J_{y}] - \mathbb{K}\cdot \mathbb{E}[J_{y}^{\top}J_{\tilde{v}_0}] - \mathbb{K}\cdot \mathbb{E}[J_{\tilde{v}_0}^{\top}J_{y}] + \mathbb{E}[J_{\tilde{v}_0}^{\top}J_{\tilde{v}_0}]\\ 
    &= \sigma^2 \cdot \underbrace{\left(B^2\fixedK^2\cdot \mathbb{G} - B \fixedK \cdot (\mathbb{Y} + \mathbb{Y}^{\top}) + I\right)}_{\defeq \mathbb{F}}.
\end{align*}
Now $\mathbb{F} \succeq 0$. To prove Lemma~\ref{lemma:mainjacob}, it suffices that for every vector $m$ of appropriate dimensions, we have that
\[m^{\top}\left(\mathbb{F} + B^2 \cdot \mathbb{G} \right)m \geq \frac{\gamma^2 \|m\|^2}{36\kappa^{10}} .\]
To prove the above we will consider two cases. The first case is when $3|B|\gamma^{-1}\kappa \geq 1$. Noting $\kappa\geq 1$, in this case Lemma \ref{lemma:GIlemma1dim} immediately implies that
\begin{align*}
    m^{\top}\left(\mathbb{F} + B^2 \cdot \mathbb{G}\right)m \geq m^{\top}\left(B^2 \cdot \mathbb{G}\right)m \geq \frac{\frac{1}{4} \|m\|^2}{9\gamma^{-2}\kappa^2} \geq \frac{\gamma^2 \|m\|^2}{36\kappa^{10}},
\end{align*}
In the second case (when $3|B|\gamma^{-1}\kappa \leq 1$), \eqref{eqn:Ynorm1dim} implies that
\begin{align*}
    m^{\top}\left(\mathbb{F} + B^2 \cdot \mathbb{G} \right)m &\geq m^{\top}\left(I - B\fixedK \cdot (\mathbb{Y} + \mathbb{Y}^{\top}) \right)m \geq (1/3)\|m\|^2 \geq \frac{\gamma^2 \|m\|^2}{36\kappa^{10}}.
\end{align*}

\subsubsection{Proof of Lemma~\ref{lemma:GIlemma1dim}}
\noindent Define the following matrix for any complex number $|\psi| < 1$.
\[\mathbb{G}(\psi) = \sum_{k_1=1}^{\hor}\sum_{k_2=1}^{\hor} \mathcal{T}_{k_1 - k_2} \left(\psi^{\dagger}\right)^{k_1-1}\psi^{k_2-1}\]
Note that $\mathbb{G}$ in Lemma \ref{lemma:GIlemma1dim} is equal to  $\mathbb{G}(A - B\fixedK)$. The following lemma provides a lower bound on the spectral properties of the matrix $\mathbb{G}(\psi)$. The lemma presents the proof of a more general case ($\phi$ is complex) that while unnecessary in the one dimensional case, aids the multi-dimensional case. A special case when $\phi = 1$ was proven in \cite{pmlr-v80-goel18a}, and we follow a similar approach relying on the inverse of such matrices. 

\begin{lemma}
\label{lemma:mainmatrixlem}
Let $\psi$ be a complex number such that $|\psi| \leq 1$. Furthermore let $\mathcal{T}_m$ is defined as an $\hor \times \hor$ matrix with $[\mathcal{T}_m]_{ij} = 1$ if and only if $i-j = m$ and $0$ otherwise. Define the matrix $\mathbb{G}(\psi)$ as
\[\mathbb{G}(\psi) = \sum_{k_1=1}^{\hor}\sum_{k_2=1}^{\hor} \mathcal{T}_{k_1 - k_2} \left(\psi^{\dagger}\right)^{k_1-1}\psi^{k_2-1}.\]
We have that
\[\mathbb{G}(\psi) \succeq (1/4)\cdot I_{H}\]
\end{lemma}

\subsubsection{Proof of Lemma \ref{lemma:mainmatrixlem}}
\begin{proof}[Proof of Lemma \ref{lemma:mainmatrixlem}]
The following definitions help us express the matrix $\mathbb{G}$ in a more convenient form. For any number $\psi \in \mathbb{C}$, such that $|\psi| < 1$ and any $h$ define,
\[S_{\psi}(h) = \sum_{i=1}^{h} |\psi|^{2(i-1)} = \frac{1 - |\psi|^{2h}}{1 - |\psi|^2}.\]
With the above definition it can be seen that the entries $\mathbb{G}(\psi)$ can be expressed in the following manner,
\[ [\mathbb{G}(\psi)]_{ij} = S_{\psi}(H-|i-j|)\cdot\psi^{i-j} \qquad \text{ if } j \geq i\]
\[ [\mathbb{G}(\psi)]_{ij} = (\psi^{\dagger})^{j-i}\cdot S_{\psi}(H-|i-j|) \qquad \text{ if } i \geq j\]  
Schematically the matrix $\mathbb{G}(\psi)$ looks like

\[\mathbb{G}(\psi) = \begin{bmatrix} 
S_{\psi}(H) & S_{\psi}(H-1)\psi & S_{\psi}(H-2)\psi^{2} & . & . & S(2)\psi^{H-2} & S(1)\psi^{H-1}\\
\psi^{\dagger}S_{\psi}(H-1) & S_{\psi}(H) & S_{\psi}(H-1)\psi & . & . & S(3)\psi^{H-3} & S(2)\psi^{H-2}\\
(\psi^{\dagger})^2S_{\psi}(H-2) & \psi^{\dagger}S_{\psi}(H-1) & S_{\psi}(H) & . & . & S(4)\psi^{H-4} & S(3)\psi^{H-3}\\
. & . & . & .& .& .& . \\
. & . & . & .& .& .& . \\
(\psi^{\dagger})^{H-1}S_{\psi}(1) & (\psi^{\dagger})^{H-2}S_{\psi}(2) & (\psi^{\dagger})^{H-3}S_{\psi}(3) & . & . & \psi^{\dagger}S_{\psi}(H-1) & S_{\psi}(H)
\end{bmatrix}.\]
We analytically compute the inverse of the matrix $\mathbb{G}(\psi)$ below and bound its spectral norm. 

\begin{claim}
The inverse of $\mathbb{G}(\psi)$ has the following form. 
\[[\mathbb{G}(\psi)]^{-1} = \begin{bmatrix} 
\alpha & b & 0 & . & . & 0 & 0 & \beta^{\dagger} \\
b^{\dagger} & a & b & . & . & 0 & 0 & 0 \\
0 & b^{\dagger} & a & . & . & 0 & 0 & . \\
. & 0 & b^{\dagger} & . & . & b & 0 & . \\
. & 0 & 0 & . & . & a & b & 0 \\
0 & 0 & 0 & . & . & b^{\dagger} & a & b \\
\beta & 0 & 0 & . & . & 0 & b^{\dagger} & \alpha 
\end{bmatrix},\]
where the relevant quantities above are given by the following formula
\[ b = \frac{-\psi}{1 + |\psi|^{2H}} \qquad a = -b(\psi^{\dagger} + \psi^{-1})= \frac{1 + |\psi|^2}{1 + |\psi|^{2H}}\]
\[\beta = \frac{(1 - |\psi|^2)}{(1 - (|\psi|^2)^{H+1})}\frac{(\psi^{\dagger})^{H}\psi}{(1 + |\psi|^{2H})} \qquad \alpha = \frac{1 - (|\psi|^2)^{H+2}}{(1 - (|\psi|^2)^{H+1})(1 + (|\psi|^{2H}))}.\]
\end{claim}

Since $|\psi| < 1 $, it is easy to see that $|\alpha|,|a|\leq 2$ and $|\beta|,|b|\leq 1$. This immediately implies that $\|(\mathbb{G}(\psi))^{-1}\| \leq 4$ and therefore the lemma follows. 
\\
\\
\noindent To prove the remnant claim, the following may be verified, implying $\mathbb{G}(\psi)[\mathbb{G}(\psi)]^{-1}=I$. 
\begin{itemize}
    \item Lets first consider the diagonal entries and in particular $i=j \in [1,H-2]$. We have that
\[ \left[\mathbb{G}(\psi)[\mathbb{G}(\psi)]^{-1}\right]_{i,i} = b \cdot \psi^{\dagger} S_{\psi}(H-1) + b^{\dagger} \cdot \psi S_{\psi}(H-1) + a S_{\psi}(H) = \frac{-2|\psi|^2S_{\psi}(H-1) + (1 + |\psi|^2)S_{\psi}(H)}{1 + |\psi|^{2H}} = 1  \]
\item Lets consider the diagonal entry $(0,0)$. (The $(H,H)$ entry is the complement and hence equal to 1).
\begin{align*}
   \left[\mathbb{G}(\psi)[\mathbb{G}(\psi)]^{-1}\right]_{0,0} &= \alpha \cdot S_{\psi}(H) + b^{\dagger}\psi S_{\psi}(H-1) + \beta^{\dagger}(\psi^{\dagger})^{H-1}S_{\psi}(1) \\
   &=  \frac{(1 - (|\psi|^2)^{H+2})S_{\psi}(H) - (1 - (|\psi|^2)^{H+1})|\psi|^2 S_{\psi}(H-1) + (1 - |\psi|^2)(|\psi|^{2H})}{(1 - (|\psi|^2)^{H+1})(1 + (|\psi|^{2H}))}\\
   &= 1
\end{align*}
\item Now lets consider non diagonal entries, in particular for $j \in [1,H-2]$ and $i \in [0,H-1]$ and $i > j$. (The case with the same conditions and $j > i$ follows by replacing $\psi$ with $\psi^{\dagger}$ in the computation below)
\begin{align*}
  \left[\mathbb{G}(\psi)[\mathbb{G}(\psi)]^{-1}\right]_{i,j} &= (\psi^{\dagger})^{i-j-1} \left(b(\psi^{\dagger})^{2}S_{H-i+j-1} + b^{\dagger}S_{H-i+j+1} + a(\psi^{\dagger})S_{H-i+j}\right) \\
  &= (\psi^{\dagger})^{i-j}\left(-|\psi|^2S_{H-i+j-1} - S_{H-i+j+1} + (|\psi|^2 + 1)S_{H-i+j}\right) \\
  &= 0
\end{align*}
\item Lastly lets consider the first column, i.e. $j=0$ and $i > 0$. (The case of the last column follows as it is the complement and hence equal to 0.)
\begin{align*}
  \left[\mathbb{G}(\psi)[\mathbb{G}(\psi)]^{-1}\right]_{i,j} &= \alpha \cdot (\psi^{\dagger})^{i}S_{\psi}(H-i) + b \cdot (\psi^{\dagger})^{i-1}S_{\psi}(H-i+1) + \beta \psi^{H-i-1} S_{\psi}(i+1) = 0.
\end{align*}

\end{itemize}

\end{proof}
\section{Conclusion}
We presented two algorithms for controlling linear dynamical systems with strongly convex costs, under certain stability assumptions, with regret that scales poly-logarithmically with time. This improves state-of-the-art known regret bounds that scale as $O(\sqrt{T})$. It remains open to extend the poly-log regret guarantees to more general systems and loss functions, such as exp-concave losses, or alternatively, show that this is impossible.

\subsection*{Acknowledgements}

The authors thank Sham Kakade and Cyril Zhang for various thoughtful discussions.  Elad Hazan acknowledges funding from NSF grant \# CCF-1704860.

\bibliography{our_bib.bib}

\begin{thebibliography}{10}

\bibitem{abbasi2014tracking}
Yasin Abbasi-Yadkori, Peter Bartlett, and Varun Kanade.
\newblock Tracking adversarial targets.
\newblock In {\em International Conference on Machine Learning}, pages
  369--377, 2014.

\bibitem{a2}
Yasin Abbasi-Yadkori, Nevena Lazic, and Csaba Szepesv{\'a}ri.
\newblock Model-free linear quadratic control via reduction to expert
  prediction.
\newblock In {\em The 22nd International Conference on Artificial Intelligence
  and Statistics}, pages 3108--3117, 2019.

\bibitem{abbasi2011regret}
Yasin Abbasi-Yadkori and Csaba Szepesv{\'a}ri.
\newblock Regret bounds for the adaptive control of linear quadratic systems.
\newblock In {\em Proceedings of the 24th Annual Conference on Learning
  Theory}, pages 1--26, 2011.

\bibitem{agarwal2019online}
Naman Agarwal, Brian Bullins, Elad Hazan, Sham Kakade, and Karan Singh.
\newblock Online control with adversarial disturbances.
\newblock In {\em Proceedings of the 36th International Conference on Machine
  Learning}, pages 111--119, 2019.

\bibitem{anava2015online}
Oren Anava, Elad Hazan, and Shie Mannor.
\newblock Online learning for adversaries with memory: price of past mistakes.
\newblock In {\em Advances in Neural Information Processing Systems}, pages
  784--792, 2015.

\bibitem{arora2012online}
Raman Arora, Ofer Dekel, and Ambuj Tewari.
\newblock Online bandit learning against an adaptive adversary: from regret to
  policy regret.
\newblock In {\em Proceedings of the 29th International Conference on Machine
  Learning}, pages 1503--1510, 2012.

\bibitem{arora2018towards}
Sanjeev Arora, Elad Hazan, Holden Lee, Karan Singh, Cyril Zhang, and Yi~Zhang.
\newblock Towards provable control for unknown linear dynamical systems.
\newblock 2018.

\bibitem{cesa2006prediction}
Nicolo Cesa-Bianchi and G{\'a}bor Lugosi.
\newblock {\em Prediction, learning, and games}.
\newblock Cambridge university press, 2006.

\bibitem{cohen2018online}
Alon Cohen, Avinatan Hasidim, Tomer Koren, Nevena Lazic, Yishay Mansour, and
  Kunal Talwar.
\newblock Online linear quadratic control.
\newblock In {\em International Conference on Machine Learning}, pages
  1028--1037, 2018.

\bibitem{dean2018regret}
Sarah Dean, Horia Mania, Nikolai Matni, Benjamin Recht, and Stephen Tu.
\newblock Regret bounds for robust adaptive control of the linear quadratic
  regulator.
\newblock In {\em Advances in Neural Information Processing Systems}, pages
  4188--4197, 2018.

\bibitem{fazel2018global}
Maryam Fazel, Rong Ge, Sham~M Kakade, and Mehran Mesbahi.
\newblock Global convergence of policy gradient methods for the linear
  quadratic regulator.
\newblock In {\em International Conference on Machine Learning}, pages
  1466--1475, 2018.

\bibitem{pmlr-v80-goel18a}
Surbhi Goel, Adam Klivans, and Raghu Meka.
\newblock Learning one convolutional layer with overlapping patches.
\newblock In {\em Proceedings of the 35th International Conference on Machine
  Learning}, pages 1783--1791, 2018.

\bibitem{OCObook}
Elad Hazan.
\newblock Introduction to online convex optimization.
\newblock {\em Foundations and Trends in Optimization}, 2(3-4):157--325, 2016.

\bibitem{hazan2018spectral}
Elad Hazan, Holden Lee, Karan Singh, Cyril Zhang, and Yi~Zhang.
\newblock Spectral filtering for general linear dynamical systems.
\newblock In {\em Advances in Neural Information Processing Systems}, pages
  4634--4643, 2018.

\bibitem{hazan2017learning}
Elad Hazan, Karan Singh, and Cyril Zhang.
\newblock Learning linear dynamical systems via spectral filtering.
\newblock In {\em Advances in Neural Information Processing Systems}, pages
  6702--6712, 2017.

\bibitem{shalev2012online}
Shai Shalev-Shwartz et~al.
\newblock Online learning and online convex optimization.
\newblock {\em Foundations and Trends{\textregistered} in Machine Learning},
  4(2):107--194, 2012.

\bibitem{stengel1994optimal}
Robert~F Stengel.
\newblock {\em Optimal control and estimation}.
\newblock Courier Corporation, 1994.

\bibitem{strang1993introduction}
Gilbert Strang.
\newblock {\em Introduction to linear algebra}, volume~3.

\end{thebibliography}
\bibliographystyle{plain}

\appendix

\section*{Appendix}

\section{Proof of Theorem \ref{thm:oco_memory}}
\begin{proof}
By the standard OGD strong convexity analysis, if $\eta_t = (\lambda\cdot (t - \hor))^{-1}$, we have that
\begin{equation*}
    \sum\limits_{t=H}^T \tilde{f}_t(x_t) - \min\limits_{x \in \K} \sum\limits_{t=H}^T \tilde{f}_t(x) \leq \frac{G^2}{2\lambda}(1 + \log(T)).
\end{equation*}
In addition, we know by \eqref{eq:memlip} that, for any $t \geq H$,
\begin{align*}
\abs{f_t(x_{t-H},\dots,x_t) - f_t(x_t,\dots,x_t)} &\leq L\sum\limits_{j=1}^H\norm{x_t - x_{t-j}} \leq L\sum\limits_{j=1}^H\sum\limits_{l=1}^j\norm{x_{t-l+1} - x_{t-l}}\\
&\leq L\sum\limits_{j=1}^H\sum\limits_{l=1}^j \eta_{t-l}\norm{\grad \tilde{f}_{t-l}(x_{t-l})} \leq LH^2\eta_{t-\hor} G,
\end{align*}
and so we have that
\begin{equation*}
\abs{\sum\limits_{t=H}^T f_t(x_{t-H},\dots,x_t) - \sum\limits_{t=H}^T f_t(x_t,\dots,x_t)} \leq \frac{LH^2 G}{\lambda}(1 + \log(T)).
\end{equation*}
It follows that
\begin{align*}
    \sum\limits_{t=H}^T f_t(x_{t-H},\dots,x_t) &- \min\limits_{x \in \K} \sum\limits_{t=H}^T f_t(x,\dots,x) \leq \frac{G^2 + L\hor^2G}{\lambda}(1 + \log(T)).
\end{align*}
\end{proof}

\begin{algorithm}[t!]
\caption{OGD with Memory (OGD-M).}
\label{ogdm}
\begin{algorithmic}[1]
\STATE \textbf{Input:} Step size $\eta$, functions $\braces{f_t}_{t=m}^T$
\STATE Initialize $x_0, \dots, x_{\hor-1} \in \K$ arbitrarily.
\FOR{$t = H, \ldots, T$}
\STATE Play $x_t$, suffer loss $f_t(x_{t-H},\dots,x_t)$
\STATE Set $x_{t+1} = \Pi_{\K}\pa{x_t - \eta\grad \tilde{f}_t(x)}$
\ENDFOR
\end{algorithmic}
\end{algorithm}

\section{Proof of Lemma \ref{lem:joc}}
\begin{proof}[Proof of Theorem~\ref{lem:joc}]
Let $M^*= \arg\min_{M\in \mathcal{M}} \sum_{t=1}^T f_t(M)$, $z_t = JM_{vec,t}$ and $z^*=JM_{\text{vec}}^*$. Now, we have, as consequence of strong convexity of $c_t$, that 
\begin{align*}
\sum_{t=1}^T f_t(M_t)-\sum_{t=1}^Tf_t(M^*) &\leq \mathbb{E} \left[ \langle \nabla_z c_t(z_t),z_t-z^*\rangle - \frac{\alpha}{2} \|z_t-z^*\|^2\right].
\end{align*}
With $P=\mathbb{E}[J^T J]$, the choice of the update rule ensures that
\begin{align*}
    \|[M_{t+1}]_{\text{vec}}-M^*_{vec}\|_{P}^2 = \|[M_{t}]_{\text{vec}}-M^*_{\text{vec}}\|_{P}^2 - 2\eta_t \langle \nabla_{[M_{t}]_{\text{vec}}} f_t(M_t),[M_{t}]_{\text{vec}} - M_{\text{vec}}^* \rangle+  \eta_t^2\|\nabla_{[M_{t}]_{\text{vec}}} f_t(M_t)\|_{P^{-1}}.
\end{align*}
Observe by the application of chain rule and linearity of expectation that
\begin{align*}
\mathbb{E} [ \langle \nabla_z c_t(z_t),z_t-z^*\rangle ] &= \mathbb{E}[\langle  \nabla_z c_t(z_t), J([M_{t}]_{\text{vec}}-M_{\text{vec}}^*)\rangle ] \\
&=\langle\nabla_{[M_{t}]_{\text{vec}}} f_t(M_t), [M_{t}]_{\text{vec}} - M_{\text{vec}}^*\rangle,\\
\mathbb{E} [ \|z_t-z^*\|^2 ] &=   \|[M_{t}]_{\text{vec}}-M_{\text{vec}}^*\|_P^2.
\end{align*}
Combining these (in)equalities, we have
\begin{align*}
&\sum_{t=1}^T f_t(M_t)-\sum_{t=1}^Tf_t(M^*) \\
\leq& \sum_{t=1}^T\left( \frac{\|[M_{t}]_{\text{vec}}-M^*_{vec}\|_{P}^2-\|[M_{t+1}]_{\text{vec}}-M^*_{vec}\|_{P}^2}{2\eta_t} + \frac{\eta_t}{2} \|\nabla_{[M_{t}]_{\text{vec}}} f_t(M_t)\|_{P^{-1}}^2 \right)\\
&-\frac{\alpha}{2} \|[M_t]_{vec}-M^*_{vec}\|_P^2 \\
\leq & (2\alpha)^{-1} \max_{M\in \mathcal{M}} \|\nabla_{M_{\text{vec}}} f_t(M)\|_{P^{-1}}^2\log T
\end{align*}
\end{proof}

\section{Proof of Lemma \ref{lem:mainreductionlemma}}

Since the proof of Lemma will borrow heavily from the definitions introduced by \cite{agarwal2019online}, we restate those definitions here for convenience. Please note that some of these definitions overload our previous definitions but it will be clear from the context.  

\subsection{Definitions}

\begin{enumerate}
    \item Let $x^K_t(M_{0:t-1})$ is the state attained by the system upon execution of a non-stationary policy $\pi(M_{0:t-1}, K)$. We similarly define $u^K_t(M_{0:t-1})$ to be the action executed at time $t$. If the same policy $M$ is used across all time steps, we compress the notation to $x_t^K(M), u_t^K(M)$. Note that $x_t^{K}(0),u_t^{K}(0)$ refers to running the linear policy $K$.
    \item $\Psi_{t,i}^{K,h}(M_{t-h:t})$ is a transfer matrix that describes the effect of $w_{t-i}$ with respect to the past $h+1$ policies on the state $x_{t+1}$, formally defined below. When $M$ is the same across all arguments we compress the notation to $\Psi_{t,i}^{K,h}(M)$.
\end{enumerate}

\begin{definition}
For any $t,h \leq t,i \leq \hor + h$, define the disturbance-state transfer matrix $\Psi_{t,i}^{K,h}$ to be a function with $h+1$ inputs defined as 
\begin{equation*}
    \Psi_{t,i}^{K,h}(M_{t-h:t}) = 
    \tilde{A}_K^i \mathbf{1}_{i\leq h} + \sum_{j=0}^{h} \tilde{A}_K^j BM_{t-j}^{[i-j-1]} \mathbf{1}_{i-j \in [1,\hor]}.
\end{equation*}
\end{definition}

\begin{definition}[Surrogate State \& Surrogate Action]
Define,
\begin{align*}
y^K_{t+1}(M_{t-\hor:t}) &= \sum_{i=0}^{2\hor} \Psi_{t,i}^{K,\hor}(M_{t-\hor:t}) w_{t-i}, \\
v_{t+1}^K(M_{t-\hor:t+1}) &= -Ky^K_{t+1}(M_{t-\hor:t})+\sum_{i=1}^\hor M_{t+1}^{[i-1]} w_{t+1-i}.
\end{align*}
When $M$ is the same across all arguments we compress the notation to $y_{t+1}^K(M),v_{t+1}^K(M)$.
\end{definition}
\begin{definition}[Surrogate Cost]
\label{defn:idealcostA}
Define the surrogate cost function $f_t$ to be the cost associated with the surrogate state and surrogate action, i.e.,
\begin{equation*}
    f_{t}(M_{t-\hor-1:t}) = \mathbb{E}\left[c_t(y^K_{t}(M_{t-\hor-1:t-1}), v^K_{t}(M_{t-\hor-1:t}))\right].
\end{equation*}
When $M$ is the same across all arguments we compress the notation to $f_t(M)$.
\end{definition}
Note that this definition coincides exactly with Definition \ref{defn:idealcost} in the main text. 

\subsection{Prerequisites}

In this section we state some lemmas and theorems which were proved in \cite{agarwal2019online}. Due to consistency of definitions the proofs of these are omitted and can be found in \cite{agarwal2019online}.

\begin{lemma}[Sufficiency]\label{l:repstat}
For any two $(\kappa,\gamma)$-diagonal strongly stable matrices $K^*, K$, there exists $M_*=(M_*^{[0]}, \ldots, M_*^{[\hor-1]}) \in \M$ defined as \[M_*^{[i]} = (K - K^*)(A - BK^*)^{i}\] such that
\begin{equation*}
    \sum_{t=0}^{T} \left( c_t(x_t^K(M_*),u_t^K(M_*)) - c_t(x_t^{K^*}(0),u_t^{K^*}(0)) \right)  \leq \\ T \cdot \frac{2\CostGradBound\BigBound\PertBound \hor \kappa_B^2 \kappa^5(1 - \gamma)^{\hor}}{\gamma}.
\end{equation*}
\end{lemma}

\begin{theorem}
\label{thm:approxthm}
For any $(\kappa, \gamma)$-diagonal strongly stable $K$, any $\tau > 0$, and any sequence of policies $M_1 \ldots M_T$ satisfying $\|M_t^{[i]}\| \leq \tau(1 - \gamma)^i$, if the perturbations are bounded by $\PertBound$, we have that

\begin{equation*}
      \sum_{t=1}^T f_{t}(M_{t - \hor-1:t}) - \sum_{t=1}^T c_t(x_t^K(M_{0:t-1}),u_t^K(M_{0:t}))  \leq 2T \CostGradBound \BigBound^2 \kappa^3 (1 - \gamma)^{\hor + 1},
\end{equation*}
where
\[\BigBound \triangleq \frac{W\kappa^3(1 + \hor\BBound\tau)}{\gamma ( 1 - \kappa^2(1 - \gamma)^{\hor+1})} + \frac{\tau W}{\gamma}.\]
\end{theorem}

\subsection{Proof of Lemma \ref{lem:mainreductionlemma}}
\begin{proof}[Proof of Lemma \ref{lem:mainreductionlemma}]
Let $\BigBound$ be defined as 
\[\BigBound \triangleq \frac{W\kappa^3(1 + \hor\BBound \tau)}{\gamma ( 1 - \kappa^2(1 - \gamma)^{\hor+1})} + \frac{\BBound \kappa^3 W}{\gamma}.\]

Let $K^*$ be the optimal linear policy in hindsight. By definition $K^*$ is a $(\kappa,\gamma)$-diagonal strongly stable matrix. Using Lemma \ref{l:repstat} and Theorem \ref{thm:approxthm}, we have that
\begin{align}
\label{eqn:approxopt}
    \min_{M_* \in \M}& \left(\sum_{t=0}^{T}  f_t(M_*)\right) - \sum_{t=0}^Tc_t(x_t^{K^*}(0),u_t^{K^*}(0)) \nonumber \\
    &\leq \min_{M_* \in \M} \left( \sum_{t=0}^T c_t(x_t^K(M_*),u_t^K(M_*))\right) - \sum_{t=0}^Tc_t(x_t^{K^*}(0),u_t^{K^*}(0)) + 2T \CostGradBound \BigBound^2 \kappa^3 (1 - \gamma)^{\hor + 1} \nonumber \\
    &\leq 2T\CostGradBound\BigBound (1 - \gamma)^{\hor + 1}\left(  \frac{\PertBound \hor \kappa_B^2 \kappa^5}{\gamma} + \BigBound \kappa^3 \right).
 \end{align}

Note that by definition of $\M$, we have that
\[ \forall t \in [T], \forall i \in [\hor] \;\;\; \|M_t^{[i]}\| \leq \BBound\kappa^3(1 - \gamma)^i.\]
Using Theorem \ref{thm:approxthm} we have that
\begin{align}
\label{eqn:approxplayed}
    \sum_{t=0}^{T} c_t(x_t^{K}(M_{0:t-1}),u_t^{K}(M_{0:t-1})) - \sum_{t=0}^{T} f_t(M_{t - \hor-1:t}) \leq 2T \CostGradBound \BigBound^2 \kappa^3 (1 - \gamma)^{\hor + 1}.
\end{align}

Summing up \eqref{eqn:approxopt} and \eqref{eqn:approxplayed} and using the condition that $\hor \geq \frac{1}{\gamma}\log(T\kappa^2)$, we get the result.\qedhere
\end{proof}

\section{Proof of Strong Convexity(Lemma \ref{lem:sc}): Multi-dimensional}
\begin{proof}[Proof of Lemma \ref{lemma:mainjacob}]

\label{sec:appsc}
Building on Section \ref{sec:sc1dim}, we prove Lemma \ref{lemma:mainjacob} for multi-dimensional systems. Using the fact that $E[w_iw_j^{\top}] = 0$ for different $i,j$ and $\mathbb{E}[w_iw_i^{\top}] = \Sigma$, it can be observed that for any $k_1,k_2$ and any $d_u \times d_u$ matrix $P$, we have that
\begin{equation}
\label{eqn:mainjacob}
  \mathbb{E}[J_{\tilde{v_{k_1}}}^{\top}PJ_{\tilde{v_{k_2}}}] = \mathcal{T}_{k_1 - k_2} \otimes P \otimes \Sigma  
\end{equation}
where $\mathcal{T}_m$ is defined as an $\hor \times \hor$ matrix with $[\mathcal{T}_m]_{ij} = 1$ if and only if $i-j = m$ and $0$ otherwise. This in particular immediately gives us that for any matrix $P$,
\begin{align}
    \label{eqn:jacoby}
    &\mathbb{E}[J_y^{\top}PJ_{y}] =\left( \sum_{k_1=1}^{\hor}\sum_{k_2=1}^{\hor} \mathcal{T}_{k_1 - k_2} \otimes \left(\left(B^{\top}(A-B\fixedK)^{\top}\right)^{k_1-1}P(A-B\fixedK)^{k_2-1}B\right) \right)\otimes \Sigma \nonumber \\
    &= \left(\left(I_{H} \otimes B^{\top}\right) \underbrace{\left(\sum_{k_1=1}^{\hor}\sum_{k_2=1}^{\hor} \mathcal{T}_{k_1 - k_2} \otimes \left(\left((A-B\fixedK)^{\top}\right)^{k_1-1}P(A-B\fixedK)^{k_2-1}\right) \right)}_{\defeq \mathbb{G}_P} \left(I_{\hor} \otimes B\right)\right) \otimes \Sigma
\end{align}
Furthermore consider the following calculation
\begin{align}
    \mathbb{E}[J_{\tilde{v_0}}^{\top}\fixedK J_y] &=
    \left(\sum_{k=1}^{\hor} \mathcal{T}_{-k} \otimes \fixedK(A-B\fixedK)^{k-1}B\right) \otimes \Sigma \\
    &= \left((I_H \otimes \fixedK)\underbrace{ \left(\sum_{k=1}^{\hor} \mathcal{T}_{-k} \otimes (A-B\fixedK)^{k-1}\right)}_{\defeq \mathbb{Y}}(I_{\hor} \otimes B)\right) \otimes \Sigma
\end{align}
As before, we state the following bounds on the spectral properties of the matrices $\mathbb{G}$ and $\mathbb{Y}$ defined above. 
\begin{lemma}\label{lem:Ynorm}
\begin{equation}
\|\mathbb{Y}\| \leq \| \sum_{k=1}^{\hor}\mathcal{T}_{-k}(A-B\fixedK)^{k-1}\| \leq \gamma^{-1}\kappa^2
\end{equation}
\end{lemma}

\begin{lemma}
\label{lemma:GIlemma}
$\mathbb{G}_I$ (where $I$ represents the Identity matrix) is a symmetric positive definite matrix with
    \[\mathbb{G}_{I} \succeq \Glb \cdot I_{\hor d_x} \]
\end{lemma}
Consider the following calculations which follows by definitions. 
\begin{align*}
    \mathbb{E}[J_{v}^{\top}J_v] &= \mathbb{E}[J_{y}^{\top}\fixedK^{\top}\fixedK J_{y}] - \mathbb{E}[J_{y}^{\top}\fixedK^{\top}J_{\tilde{v}_0}] - \mathbb{E}[J_{\tilde{v}_0}^{\top}\fixedK J_{y}] + \mathbb{E}[J_{\tilde{v}_0}^{\top}J_{\tilde{v}_0}]\\ 
    &= \underbrace{\left((I_{\hor} \otimes B^{\top})\mathbb{G}_{\fixedK^{\top}\fixedK} (I_{\hor} \otimes B) - \mathbb{Y}(I_{\hor}\otimes B) - (I_{\hor}\otimes B^{\top})\mathbb{Y}^{\top} + I_{\hor d_u}\right)}_{\defeq \mathbb{F}}\otimes \Sigma
\end{align*}
Since we know that $\Sigma \succeq 0$ we immediately get that $\mathbb{F}  \succeq 0$. Using the above calculations it is enough to show that the following matrix has lower bounded eigenvalues, i.e. for every vector $m$ of appropriate dimensions, we have that
\[m^{\top}\left(\mathbb{F} + (I_{\hor} \otimes B^{\top})\mathbb{G}_{I} (I_{\hor} \otimes B)\right)m \geq \frac{\gamma^2\|m\|^2}{36\kappa^{10}} \]

To prove the above we will consider two cases. The first case is when $\|(I_{\hor} \otimes B)m\| \geq \frac{\gamma\|m\|}{3\kappa^3}$. In this case note that
\begin{align*}
    m^{\top}\left(\mathbb{F} + (I_{\hor} \otimes B^{\top})\mathbb{G}_{I} (I_{\hor} \otimes B)\right)m \geq m^{\top}\left((I_{\hor} \otimes B^{\top})\mathbb{G}_{I} (I_{\hor} \otimes B)\right)m \geq \frac{\Glb \gamma^2 \|m\|^2}{9\kappa^6}
\end{align*}
In the second case (when $\|(I_{\hor} \otimes B)m\| \leq \frac{\gamma\|m\|}{3\kappa^3}$), we have that
\begin{align*}
    m^{\top}\left(\mathbb{F} + (I_{\hor} \otimes B^{\top})\mathbb{G}_{I} (I_{\hor} \otimes B)\right)m &\geq m^{\top}\left(  I_{\hor d_u} - (I_{\hor}\otimes \fixedK)\mathbb{Y}(I_{\hor}\otimes B) - (I_{\hor}\otimes B^{\top})\mathbb{Y}^{\top}(I_{\hor}\otimes \fixedK^{\top}) \right)m \\
    &\geq (1/3)\|m\|^2 \geq \frac{\gamma^2\|m\|^2}{36\kappa^{10}}.
\end{align*}
\end{proof}
\noindent We now finish the proof with the proof of Lemmas \ref{lem:Ynorm} and \ref{lemma:GIlemma}.
\begin{proof}[Proof of Lemma \ref{lem:Ynorm}]
Since $\fixedK$ is $(\kappa, \gamma)$-diagonal strongly stable, we can diagonalize the matrix $A-B\fixedK$ as $A-B\fixedK = QLQ^{-1}$ with $\|Q\|,\|Q\|^{-1} \leq \kappa$. Therefore,
\[\mathbb{Y} = \left(\sum_{k=1}^{\hor} \mathcal{T}_{-k} \otimes QL^{k-1}Q^{-1}\right) = (I_{\hor} \otimes Q)\left(\sum_{k=1}^{\hor} \mathcal{T}_{-k} \otimes L^{k-1}\right)(I_{\hor} \otimes Q^{-1}).\]
Now consider the matrix $P$ for any complex number $\phi$ with $|\phi| < 1$.
\[P = \sum_{k=1}^{H} \mathcal{T}_{-k} \phi^{k-1} \]
We wish to bound $\|P\|$. To this end consider $PP^{\top}$ and consider the $\ell_1$ norm of any row. It can easily be seen that the $\ell_1$ norm of any row of $PP^{\top}$ is bounded by $\frac{1}{1 - |\phi|}\cdot \frac{1}{1 - |\phi|^2}$, and therefore
\[\|P\| = \sqrt{\|PP^{\top}\|} \leq \sqrt{\frac{1}{(1 - |\phi|)(1 - |\phi|^2)}}.\]
Using that $L$ is diagonal with entries bounded in magnitude by $1 - \gamma$, we get that $\|\mathbb{Y}\| \leq \gamma^{-1}\kappa^2$.
\end{proof}

\begin{proof}[Proof of Lemma \ref{lemma:GIlemma}]
We need to consider the following matrix 
\[\mathbb{G}_I = \sum_{k_1=1}^{\hor}\sum_{k_2=1}^{\hor} \mathcal{T}_{k_1 - k_2} \otimes \left(\left((A-B\fixedK)^{\top}\right)^{k_1-1}(A-B\fixedK)^{k_2-1}\right)\]
Since $\fixedK$ is $(\kappa, \gamma)$-diagonal strongly stable, we can diagonalize the matrix $A-B\fixedK$ as $A-B\fixedK = QLQ^{-1}$ with $\|Q\|,\|Q\|^{-1} \leq \kappa$. Further since $A-B\fixedK$ is a real valued matrix we have that $(A-B\fixedK)^{\top} = (Q^{-1})^{\dagger}L^{\dagger}Q^{\dagger}$. Therefore we have that
\[\mathbb{G}_I = \sum_{k_1=1}^{\hor}\sum_{k_2=1}^{\hor} \mathcal{T}_{k_1 - k_2} \otimes \left((Q^{-1})^{\dagger}\left(L^{\dagger}\right)^{k_1-1}Q^{\dagger}QL^{k_2-1}Q^{-1}\right)\]
Further consider the following matrix $\hat{\mathbb{G}}$.
\[ \hat{\mathbb{G}} = \begin{bmatrix} 
0 & 0 & . & . & I\\
0 & . & . & . & L\\
. & . & . & . & L^2\\
. & 0 & . & . & .\\
0 & I & . & . & .\\
I & L & . & . & L^{H-1}\\
L & L^2 & . & . & 0\\
L^2 & . & . & . & 0\\
. & . & . & . & .\\
. & L^{H-1} & . & . & .\\
L^{H-1} & 0 & . & . & 0
\end{bmatrix}\]
It can be seen that, 
\begin{equation}
\label{eqn:sc1}
\left(( I_{2H-1}  \otimes Q)\hat{\mathbb{G}}(I_{2H-1} \otimes Q^{-1})\right)^{\dagger}\left(( I_{2H-1}\otimes Q)\hat{\mathbb{G}}( I_{2H-1} \otimes Q^{-1})\right) = \mathbb{G}_I.    
\end{equation}

Furthermore note that since $\|Q\|,\|Q^{-1}\| \leq \kappa$, therefore all singular values of $Q$ lie in the range $[\kappa^{-1},\kappa]$. Therefore it follows that
\begin{equation}
\label{eqn:sc2}
    Q^{\dagger}Q \succeq \kappa^{-2}I \qquad (Q^{-1})^{\dagger}Q^{-1} \succeq \kappa^{-2}I
\end{equation}
Using \eqref{eqn:sc1},\eqref{eqn:sc2} it follows that
\begin{equation}
    \label{eqn:sc3}
    \mathbb{G}_I \succeq \kappa^{-4}\cdot \left(\hat{\mathbb{G}}\right)^{\dagger}\left(\hat{\mathbb{G}}\right)
\end{equation}

Therefore we only need to show that $\left(\hat{\mathbb{G}}\right)^{\dagger}\left(\hat{\mathbb{G}}\right)$ has a lower bounded eigenvalue. To that end notice that since $L$ is a diagonal matrix with diagonal values whose magnitude is upper bounded by 1. Therefore, it sufficient to consider the case when $L$ is a scalar complex number with magnitude upper bounded by 1. To this end we can consider the following simplification of $\mathbb{G}_{I}$ defined for a complex number $\psi$ with $|\psi| < 1$ as defined earlier.

\[\mathbb{G}(\psi) = \sum_{k_1=1}^{\hor}\sum_{k_2=1}^{\hor} \mathcal{T}_{k_1 - k_2} \left(\psi^{\dagger}\right)^{k_1-1}\psi^{k_2-1}\]
Invoking Lemma \ref{lemma:mainmatrixlem} we immediately get that
\[\mathbb{G}_I \succeq \kappa^{-4}\cdot \left(\hat{\mathbb{G}}\right)^{\dagger}\left(\hat{\mathbb{G}}\right) \succeq \frac{1}{4\kappa^4} \cdot I_{Hd_x}.\]
\end{proof}

\end{document}